\documentclass{article}

\usepackage[accepted]{icml2023}

\usepackage[colorlinks=true,allcolors=perfblue]{hyperref} 
\usepackage[utf8]{inputenc} %
\usepackage[T1]{fontenc}    %
\usepackage{url}            %
\usepackage{booktabs}       %
\usepackage{amsfonts}       %
\usepackage{nicefrac}       %
\usepackage{microtype}      %
\usepackage{lipsum}		%
\usepackage{graphicx}
\usepackage{natbib}
\usepackage{doi}
\usepackage{amsmath}
\usepackage{algorithm}
\usepackage{algorithmic}
\usepackage{amssymb}
\usepackage{pifont}
\usepackage{tikz}
\usetikzlibrary{arrows,calc} 
\newcommand{\tikzAngleOfLine}{\tikz@AngleOfLine}
\def\tikz@AngleOfLine(#1)(#2)#3{%
\pgfmathanglebetweenpoints{%
\pgfpointanchor{#1}{center}}{%
\pgfpointanchor{#2}{center}}
\pgfmathsetmacro{#3}{\pgfmathresult}%
}

\usepackage[utf8]{inputenc} %
\usepackage[T1]{fontenc}    %
\usepackage{url}            %
\usepackage{booktabs}       %
\usepackage{amsfonts, amsmath, amssymb, bm}       %
\usepackage{nicefrac}       %
\usepackage{microtype}      %
\usepackage{xcolor}         %
\usepackage{amsthm}
\usepackage{thmtools}
\usepackage{graphicx}
\usepackage{float}
\usepackage{algorithm}
\usepackage{algorithmic}

\usepackage{caption}
\usepackage{subcaption}

\DeclareMathOperator*{\argmax}{\arg\!\max}

\usepackage{dsfont}

\usepackage{xcolor}
\definecolor{expert}{HTML}{008000}
\definecolor{error}{HTML}{f96565}
\definecolor{learner}{HTML}{F79646}
\definecolor{perfblue}{RGB}{64, 114, 175}

\usepackage{etoc}

\usepackage{color-edits}
\addauthor{sw}{blue}

\theoremstyle{plain}
\newtheorem{theorem}{Theorem}[section]

\newtheorem{lemma}[theorem]{Lemma}
\newtheorem{corollary}[theorem]{Corollary}
\theoremstyle{definition}
\newtheorem{definition}[theorem]{Definition}

\theoremstyle{remark}

\declaretheoremstyle[
headfont=\normalfont\itshape,
qed=\qedsymbol,
]{mypf}

 \renewcommand{\algorithmiccomment}[1]{// #1}

\usepackage{transparent}
\newcommand{\algcomment}[1]{\textcolor{perfblue}{\transparent{0.8}\small{\texttt{\textbf{//\hspace{2pt}#1}}}}}

\icmltitlerunning{Inverse Reinforcement Learning without Reinforcement Learning}

\begin{document}
\twocolumn[
\icmltitle{Inverse Reinforcement Learning without Reinforcement Learning}

\icmlsetsymbol{equal}{*}

\begin{icmlauthorlist}
\icmlauthor{Gokul Swamy}{cmu}
\icmlauthor{Sanjiban Choudhury}{cornell,aurora}
\icmlauthor{J. Andrew Bagnell}{aurora,cmu}
\icmlauthor{Zhiwei Steven Wu}{cmu}
\end{icmlauthorlist}

\icmlaffiliation{cmu}{Carnegie Mellon University}
\icmlaffiliation{cornell}{Cornell University}
\icmlaffiliation{aurora}{Aurora Innovation}

\icmlcorrespondingauthor{Gokul Swamy}{gswamy@cmu.edu}

\icmlkeywords{Machine Learning, ICML}

\vskip 0.3in
]

\printAffiliationsAndNotice{}  %

\begin{abstract}
Inverse Reinforcement Learning (IRL) is a powerful set of techniques for imitation learning that aims to learn a reward function that rationalizes expert demonstrations. Unfortunately, traditional IRL methods suffer from a computational weakness: they require repeatedly solving a hard reinforcement learning (RL) problem as a subroutine. This is counter-intuitive from the viewpoint of reductions: we have reduced the \textit{easier} problem of imitation learning to repeatedly solving the \textit{harder} problem of RL. Another thread of work has proved that access to the side-information of the distribution of states where a strong policy spends time can dramatically reduce the sample and computational complexities of solving an RL problem. In this work, we demonstrate for the first time a more informed imitation learning reduction where we utilize the state distribution of the expert to alleviate the global exploration component of the RL subroutine, providing an \textit{exponential} speedup in theory. In practice, we find that we are able to significantly speed up the prior art on continuous control tasks.
\end{abstract}

\section{Introduction}
 Inverse Reinforcement Learning (IRL), also known as Inverse Optimal Control \citep{kalman1964linear, bagnell2015invitation} or Structural Estimation \citep{rust1994structural}, is the problem of finding a reward function that \textit{rationalizes} (i.e. makes optimal) demonstrated behavior. Such approaches build on the lengthy history of trying to understand intelligent behavior \citep{muybridge1887animal} as approximate optimization of some cost function \citep{wolpert1995internal}. While economists \citep{rust1994structural} and cognitive scientists \citep{baker2009action} are often interested in analyzing the recovered reward function, it is more common in machine learning to view IRL algorithms as methods to \textit{imitate} \citep{ziebart2008maximum} or \textit{forecast} \cite{kitani2012activity} expert behavior. 
 
There are three key benefits to the IRL approach to imitation. The first is \textit{policy space structuring}: effectively, IRL reduces our (often large) policy class to just those policies that are (approximately) optimal under some member of our (relatively small) reward function class. %
\begin{figure}
     \centering
     \includegraphics[width=0.45\textwidth]{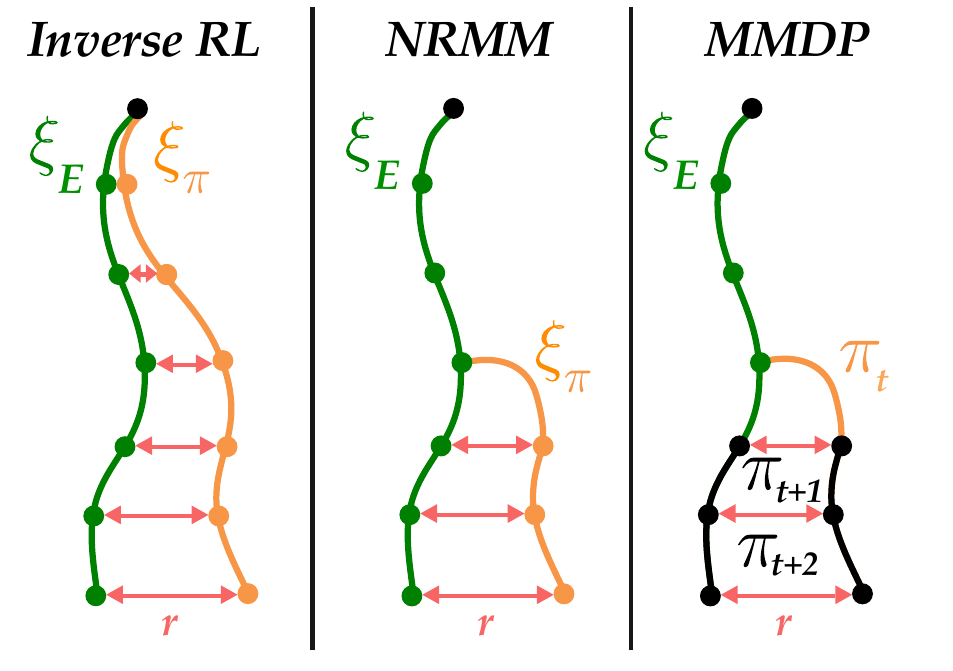}
     \caption{Traditional Inverse RL methods (left) repeatedly solve RL problems with adversarially chosen rewards in their inner loop which can be rather computationally expensive. We introduce two \textit{exponentially} faster methods for IRL. \texttt{NRMM} (No-Regret Moment Matching, center) resets the learner to states from the expert demonstrations before comparing trajectory suffixes. \texttt{MMDP} (Moment Matching by Dynamic Programming, right) optimizes a sequence of policies backwards in time. Both methods avoid solving the global exploration problem inherent in RL.}
     \label{fig:ffig}
 \end{figure}
The second is \textit{transfer across problems}: for many practical applications (e.g.\ robotics \citep{silver2010learning, ratliff2009learning, kolter2008control, ng2006autonomous, zucker2011optimization}, computer vision \citep{kitani2012activity}, and human-computer interaction \citep{ziebart2008navigate, ziebart2012probabilistic}), one is able to learn a \textit{single} reward function across multiple instances and then use it to forecast or imitate expert behavior in new problems that arise at test time. As \citet{ng2000algorithms} put it, ``the entire field of reinforcement learning is founded on the presupposition that the reward function, rather than the policy is the most succinct, robust, and \textit{transferable} definition of the task'' (italics ours). The third is \textit{robustness to compounding errors}: as IRL methods involve the learner performing rollouts in the environment, they cannot end up in states they didn't expect to at test time and therefore will not suffer from compounding errors \citep{swamy2021moments}. Taken together, these three reasons help explain why IRL methods continue to provide state-of-the-art results on challenging imitation learning problems (e.g. in autonomous driving \citep{bronstein2022hierarchical, igl2022symphony, vinitsky2022nocturne}).

The most widely used approaches to IRL \citep{ziebart2008maximum, ho2016generative} are fundamentally game-theoretic. An RL algorithm \textit{generates} trajectories by optimizing (i.e. \textit{decoding}) the current reward function. In response, a reward function selector picks a new reward function that \textit{discriminates} between learner and expert trajectories. As pointed out by \citet{finn2016guided}, the IRL setup generalizes a GAN \citep{goodfellow2020generative} with a dynamics model in the generation stem. More specifically, if one looks at the typical structure of an IRL algorithm, one performs the decoding-via-RL operation repeatedly in an \textit{inner loop}, tweaking the current estimate of the reward function in the \textit{outer loop} to produce behavior that more closely resembles that of the expert.

For some problems, highly optimized planners \citep{ratliff2009learning} or optimal controller synthesis procedures \citep{levine2013guided} allow an efficient implementation of this inner loop. More generally however, one might want to tackle problems that don't have efficient algorithms to decode behavior and therefore be forced to rely on sample-based RL algorithms. Unfortunately, this can make each inner loop iteration quite inefficient (both in terms of computational and sample efficiency) as it requires solving the \textit{global} exploration problem inherent in RL. From the lens of reductions \cite{beygelzimer2009tutorial}, such an approach is counter-intuitive as we've turned the relatively easy problem of imitating an expert into the repeated solving of the hard problem of RL.

Prior work \citep{kakade2002approximately, bagnell2003policy, ross2011} has shown that access to a good \textit{exploration distribution} (i.e. the states where a strong policy spends much of its time) can dramatically reduce the complexity of RL as the learner doesn't have to explore for as long: knowing a set of waypoints along the shortest path through a maze should speed up your attempt to solve it. In the imitation learning setup, we have access to just such a distribution: \textit{the expert's visitation distribution}. Our key insight is that \textbf{\textit{expert demonstrations can dramatically improve the efficiency of the RL subroutine of IRL}}. %

More explicitly, our contributions are as follows:

\noindent \textbf{1. We derive two algorithms that reset the learner to states from the expert visitation distribution for more efficient IRL.} \texttt{MMDP} (Moment Matching by Dynamic Programming) produces a sequence of policies. \texttt{NRMM} (No-Regret Moment Matching) produces a single, stationary policy. Both come in \textit{primal} and \textit{dual} variants.%

\noindent \textbf{2. We discuss the statistical complexity of expert resets.} We prove that in the worst case, traditional IRL algorithms take an \textit{exponential} number of interactions (in the horizon of the problem) to learn a policy competitive with the expert. In contrast, we prove that our algorithms require only \textit{polynomial} interactions per iteration to learn policies competitive with the expert. 

\noindent \textbf{3. We discuss the performance implications of expert resets.} We show that in the worst case, neither \texttt{MMDP} nor \texttt{NRMM} can avoid a quadratic compounding of errors with respect to the horizon. 

\noindent \textbf{4. We derive a practical meta-algorithm that achieves the best of both.} We propose \texttt{FILTER} (Fast Inverted Loop Training via Expert Resets) which interpolates between traditional IRL and our own approaches via mixing expert resets with standard resets. This allows use to ease the exploration burden on the learner while mitigating compounding errors. We implement \texttt{FILTER} on continuous control tasks and find it is more efficient than standard IRL.

We begin with a discussion of related work.

\section{Related Work}
Both of our algorithms build upon prior work in utilizing strong exploration distributions in the reinforcement learning context \citep{kakade2002approximately}. In a sense, we lift these insights to the imitation learning context. \texttt{MMDP} can be seen as the moment-matching version of the PSDP (Policy Search by Dynamic Programming) algorithm of \citet{bagnell2003policy}. \texttt{NRMM} calls the NRPI (No-Regret Policy Iteration) algorithm of \citet{ross2011} in each iteration. Recent work by \citet{uchendu2022jump} has confirmed that PSDP and NRPI continue to provide strong computational benefits with modern training algorithms and architectures, boding well for their application to IRL.

Our work is also related to recent developments in the theory of policy gradient algorithms by \citet{agarwal2021theory}, in that we also assume access to a reset distribution that covers the visitation distribution of the policy we compare the learner's to. While they compare to the optimal policy, we compare to the expert, as we are focused IRL.

\citet{swamy2022minimax} also consider sample efficiency in IRL, but focus on making the most out of a finite set of expert demonstrations, rather than solving the moment-matching problem with as few learner-environment interactions as possible. Our work is therefore complementary to theirs. 

Perhaps the most similar algorithm to \texttt{MMDP} is the FAIL algorithm of \citet{sun2019provably}. While both algorithms solve a sequence of moment-matching games, they differ in several key ways. Perhaps most obviously, FAIL is solving the sequence of games forward in time while \texttt{MMDP} is solving them backwards in time. This makes it straightforward to mix \texttt{MMDP} with value-based reinforcement learning (which also uses backwards-in-time dynamic programming), while it is not apparent how to do so for FAIL. %

\begin{algorithm}[t]
\begin{algorithmic}
\STATE {\bfseries Input:} Demos. $\mathcal{D}_E$, Policy class $\Pi$, Reward class $\mathcal{F}_r$
\STATE {\bfseries Output:} Trained policy $\pi$
\STATE Initialize $\pi_1 \in \Pi$
\FOR{$i$ in $1 \dots N$}
\STATE \algcomment{Use any no-regret algo to pick $f$}
\STATE $f_{i} \gets \underset{f \in \mathcal{F}_r}{\argmax } J(\pi_E, f) - J(\text{Unif}(\pi_{1:i}), f)+ R(f)$
\STATE $\pi_{i+1} \gets \texttt{MaxEntRL}(r=f_i)$

\ENDFOR
\STATE {\bfseries Return } $\pi_i$ with the lowest validation error.
\end{algorithmic}
\caption{\texttt{IRL} (Dual, \citet{ziebart2008maximum}) \label{alg:dual-irl}}
\end{algorithm}

An alternative way to use expert demonstrations in policy optimization is via regularization towards a trained behavioral cloning policy \citep{jacob2022modeling, tiapkin2023regularized}. The benefits of such a technique are quite problem-specific (e.g. such regularization could introduce compounding errors where none would exist otherwise). However, on problems where such regularization is helpful, it can easily be combined with the improved efficiency our techniques provide. %

Another line of work attempts to improve the efficiency of IRL algorithms by learning $Q$ functions and then differencing them across sequential states to extract a reward function, eliding the need for an inner loop \citep{dvijotham2010inverse, garg2021iq}. While reward functions don't include information about the dynamics of the environment, $Q$ functions do. This means that $Q$ function-based approaches to IRL need to spend computation and samples to learn the environment's dynamics, only to immediately filter them out, which seems rather inefficient and can introduce other errors from the secondary regression. Furthermore, while one might be able to apply traditional IRL methods on datasets collected from diverse agents solving tasks with similar goals \citep{silver2010learning, ratliff2009learning, kolter2008control, ng2006autonomous, zucker2011optimization, ziebart2008navigate, ziebart2012probabilistic}, it isn't clear that $Q$ function-based approaches would output consistent estimates of the expert's reward function if dynamics differ across environments.

\section{Expert Resets in Inverse RL}
We utilize the moment-matching framework of \citet{swamy2021moments} to prove performance bounds for our algorithms. Our results allow one to speed up \textit{any} member of the broad reward-moment-matching phylum of their taxonomy that uses RL (e.g. MaxEnt IRL \citep{ziebart2008maximum}, GAIL \citep{ho2016generative}, SQIL \citep{reddy2019sqil}).

\begin{algorithm}[t]
\begin{algorithmic}
\STATE {\bfseries Input:} Demos. $\mathcal{D}_E$, Policy class $\Pi$, Reward class $\mathcal{F}_r$
\STATE {\bfseries Output:} Trained policy $\pi$
\STATE Initialize $f_1 \in \mathcal{F}_r$
\FOR{$i$ in $1 \dots N$}
\STATE \algcomment{Use any no-regret algo to pick $\pi$}
\STATE $\pi_i \gets \texttt{MaxEntRL}(r = \frac{1}{i}\sum_{j=1}^i f_j)$
\STATE $f_{i+1} \gets \underset{f \in \mathcal{F}_r}{\argmax} J(\pi_E, f) - J(\pi_i, f)$
\ENDFOR
\STATE {\bfseries Return } $\pi_i$ with the lowest validation error.
\end{algorithmic}
\caption{\texttt{IRL} (Primal, \citet{syed2007game}) \label{alg:primal-irl}}
\end{algorithm}

\subsection{Inverse RL as (Inefficient) Game Solving}
Consider a finite-horizon Markov Decision Process (MDP) \citep{puterman2014markov} parameterized by $\langle \mathcal{S}, \mathcal{A}, \mathcal{T}, r, T \rangle$ where $\mathcal{S}$, $\mathcal{A}$ are the state and action spaces, $\mathcal{T}: \mathcal{S} \times \mathcal{A}\rightarrow \Delta(\mathcal{S})$ is the transition operator, $r: \mathcal{S} \times \mathcal{A} \rightarrow [-1, 1]$ is the reward function, and $T$ is the horizon. In the inverse RL setup, we see trajectories generated by an expert policy $\pi^E: \mathcal{S} \rightarrow \Delta(\mathcal{A})$, but do not know the reward function. Our goal is to nevertheless learn a policy that performs as well as the expert's, no matter the true reward function.

\begin{algorithm*}[t]
\begin{algorithmic}
\STATE {\bfseries Input:} Sequence of expert visitation distributions $\rho_{E}^1 \dots \rho_{E}^T$, Policy class $\Pi$, Reward class $\mathcal{F}_r$
\STATE {\bfseries Output:} Sequence of trained policies $\pi = \pi^{1:T}$
\FOR{$t$ in $T \dots 1$}
\STATE Set $\mathcal{D}_t = \{\}$
\FOR{$j=1$ to $M$}
\STATE Sample a random start state $s_t \sim \rho_{E}^t$.
\STATE Execute a random action $a_t \sim \text{Unif}(\mathcal{A})$ in $s_t$.
\STATE Follow $\pi^{t+1:T}$ until the end of the horizon.
\STATE $\mathcal{D}_t \gets \mathcal{D}_t \cup \{(s_t, a_t, s_{t+1:T}, a_{t+1:T})\}$
\ENDFOR
\STATE \algcomment{(Approximately) solve moment-matching game.}
\begin{equation}
    \pi^t \gets \arg\min_{\pi \in \Pi} \max_{f \in \mathcal{F}_r} \frac{1}{T} \left( \mathbb{E}_{\substack{s_t \sim \rho_{E}^t, \\ a_t \sim  \rho_E^t(s_t)}}\left[\mathbb{E}_{\mathcal{D}_t \vert s_t, a_t}\left[\sum_{\tau = t}^T f(s_{\tau}, a_{\tau})\right]\right] - \mathbb{E}_{\substack{s_t \sim \rho_{E}^t, \\ a_t \sim \pi(s_t)}}\left[\mathbb{E}_{\mathcal{D}_t \vert s_t, a_t}\left[\sum_{\tau = t}^T f(s_{\tau}, a_{\tau})\right]\right] \right) \label{eq:mmdp-payoff}
\end{equation}
\ENDFOR
\STATE {\bfseries Return } $\pi^{1:T}$.
\end{algorithmic}
\caption{\texttt{MMDP} (Moment Matching by Dynamic Programming): Primal \label{alg:mmdp}}
\end{algorithm*}

We solve the IRL problem via equilibrium computation between a policy player and an adversary that tries to pick out differences between expert and learner policies along certain moments (i.e. potential components of the reward function) \cite{swamy2021moments}. More formally, we optimize over (time-varying) policies $\pi = \{\pi_1, \dots, \pi_T\}$, with $\pi_t: \mathcal{S} \rightarrow \Delta(\mathcal{A}) \in \Pi$ and reward functions $f: \mathcal{S} \times \mathcal{A} \rightarrow [-1, 1] \in \mathcal{F}_r$. For simplicity, we assume that our strategy spaces ($\Pi$ and $\mathcal{F}_r$) are convex and compact, that $\mathcal{F}_r$ is closed under negation, and that $r \in \mathcal{F}_r, \pi_E \in \Pi$. \footnote{If we do not assume realizability, we would get the analogous agnostic bounds throughout the following sections.} We solve (i.e. compute an approximate Nash equilibrium) of the two-player zero sum game
\begin{equation}
    \min_{\pi \in \Pi} \max_{f \in \mathcal{F}_r} J(\pi_E, f) - J(\pi, f),
\end{equation}
where $J(\pi, f) = \mathbb{E}_{\xi \sim \pi}[\sum_{t=0}^T f(s_t, a_t)]$ denotes the value of policy $\pi$ under reward function $f$.

\citet{swamy2021moments} describe two different classes of strategies for equilibrium computation: \textit{primal}, where the policy player follows a no-regret strategy against a best-response discriminative player and \textit{dual}, where the discriminative player follows a no-regret strategy against a best-response policy player. Most IRL algorithms are dual (e.g. MaxEnt IRL \citep{ziebart2008maximum} or LEARCH \citep{ratliff2009learning}) but there do exist primal approaches (e.g. MWAL \citep{syed2007game}, GAIL \citep{ho2016generative}). For both classes of strategies, a best-response corresponds to an inner loop iteration, while a no-regret step corresponds to an outer loop iteration.

For the policy player, a best-response consists of solving the RL problem under the current adversarially chosen reward function, i.e.
\begin{equation}
   \pi_{i+1} = \argmax_{\pi \in \Pi} J(\pi, f_i) + H(\pi),
\end{equation}
while a no-regret step consists of running \textit{any} no-regret online learning algorithm over the history of rewards, \footnote{We write down a specific no-regret algorithm here (Follow the Regularized Leader \citep{mcmahan2011follow}) but one could use any other (e.g. Multiplicative Weights \citep{arora2012multiplicative} or Online Gradient Descent \citep{zinkevich2003online}) and have similar guarantees.} e.g. 
\begin{equation}
   \pi_{i+1} = \argmax_{\pi \in \Pi} J(\pi, \frac{1}{i}\sum_{j=0}^i f_j) + H(\pi),
\end{equation}
where $H(\pi)$ denotes the entropy of the policy. See Algorithms \ref{alg:dual-irl} and \ref{alg:primal-irl} for psuedocode, with $R(f)$ being a strongly convex regularizer. 

In both cases, one is solving a full RL problem at each iteration. This means that in the worst case, one pays exponentially in the horizon at each iteration \citep{kakade2003sample}:
\begin{theorem}{\textbf{Inverse RL Sample Complexity:}} For Algorithms \ref{alg:dual-irl} and \ref{alg:primal-irl}, there exists an MDP, $\pi_E$, $\Pi$, and $\mathcal{F}_r$ such that returning a policy $\pi$ which satisfies $J(\pi_E, r) - J(\pi, r) \leq 0.5 V_{max}$ requires $\Omega(|\mathcal{A}|^T)$ interactions with the environment, where $V_{max}$ is the value of the optimal policy. \hyperref[pf:thm:irl-exp]{[Proof]}
\label{thm:irl-exp}
\end{theorem}
We now discuss how we can utilize the (already known from the demonstrations) expert's visitation distribution to solve RL problems more efficiently.

\subsection{Method 1: Dynamic Programming}

Dynamic programming in the form of the Bellman Equation forms the basis of $Q$-learning based approaches to RL: one "backs-up" $Q$ values backwards-in-time, selecting actions based on the sum of the reward at the current timestep and the already computed value of the next state. More generally however, one can back-up \textit{policies} rather than just $Q$-values, as in the Policy Search by Dynamic Programming (PSDP) algorithm of \citet{bagnell2003policy}. Given some roll-in distribution $\nu$, the algorithm draws states from timestep $T$ and selects a policy 
\begin{equation}
    \pi_T = \argmax_{\pi \in \Pi} \mathbb{E}_{s \sim \nu^T}[r(s, \pi(s))].
\end{equation}
Then, holding this policy fixed, the algorithm draws states from the roll-in distribution at timestep $T-1$ and selects a policy for timestep $T-1$ that maximizes reward over the horizon,
\begin{equation}
    \pi_{T-1} = \argmax_{\pi \in \Pi} \mathbb{E}_{s \sim \nu^{T-1}}[r(s, \pi(s)) + r(s', \pi_T(s'))],
\end{equation}
where $s'$ denotes a successor state. This induction proceeds backwards in time until one reaches the first timestep, at which point a sequence of policies $\pi_{1:T}$ is output. Notice that at each step of this algorithm, we are solving a single-step classification problem. So, instead of the exponential-in-the-horizon complexity one must pay (in hard instances) for RL, one pays only \textit{quadratically} in the horizon. 

The careful reader will notice that PSDP requires a reward function. Two strategies come to mind for adversarially picking one for IRL. The first is to choose a reward for each \textit{timestep} (i.e. each $t \in [T]$) of PSDP. The second is to run PSDP to completion (i.e. solve for all $T$ policies) and then pick a new reward in an outer loop. We focus on the first, \textit{primal} strategy in the main text and defer the latter, \textit{dual} strategy to Appendix \ref{app:dual} for space reasons. We call the resulting algorithm \texttt{MMDP}: \textit{Moment Matching by Dynamic Programming} and outline the procedure in Algorithm \ref{alg:mmdp}.
Throughout our analysis, we define optimization error $\epsilon_t$ as the value when $\pi_t$ is plugged into Eq. \eqref{eq:mmdp-payoff}. Like PSDP, \texttt{MMDP} avoids the exponential sample complexity of RL.
\begin{lemma}{\textbf{\texttt{MMDP} Sample Complexity:}} Let $\epsilon > 0$. At iteration $t$, \texttt{MMDP} requires at most $$O\left(\log\left(\frac{|\Pi||\mathcal{F}_r|}{\delta}\right) \frac{T^3|\mathcal{A}|^2}{\epsilon^2}\right)$$ interactions with the environment to, w.p. $\geq 1 - \delta$, produce a policy $\pi_t$ with optimization error $\epsilon_t \leq \epsilon$ (Eq. \ref{eq:mmdp-payoff}). \hyperref[pf:thm:mmdp-sc]{[Proof]}
\label{thm:mmdp-sc}
\end{lemma}
\texttt{MMDP} performs $T$ iterations, giving us an overall complexity that is still polynomial in the relevant quantities. \footnote{For simplicity, we consider finite classes. One could instead use another complexity measure (e.g. Rademacher) that extends to classes with infinite elements \citep{sriperumbudur2009integral}.} We prove the following performance bound on the policies produced by \texttt{MMDP} in Appendix \ref{app:proofs}:
\begin{theorem}{\textbf{\texttt{MMDP} Upper Bound:}}
    Let $\pi$ denote the sequence of policies returned by \texttt{MMDP} and $\Bar{\epsilon} = \frac{1}{T}\sum_t^T \epsilon_t$, where $\epsilon_t$ denotes the optimization error of $\pi_t$ (Eq. \ref{eq:mmdp-payoff}). Then,
    \begin{equation}
        J(\pi_E) - J(\pi) \leq \Bar{\epsilon}T^2
    \end{equation}
    \hyperref[pf:thm:mmdp]{[Proof]}
    \label{thm:mmdp}
\end{theorem}

This bound tells us how training error $\Bar{\epsilon}$ translates to our policy's test-time performance. The lower bound matches, making the above tight.
\begin{theorem}{\textbf{\texttt{MMDP} Lower Bound:}}
    There exists an MDP, $\pi_E$ and sequence of policies $\pi$ with $\Bar{\epsilon} = \frac{1}{T}\sum_t^T \epsilon_t$, where $\epsilon_t$ denotes the optimization error of $\pi_t$ (Eq. \ref{eq:mmdp-payoff}), such that
    \begin{equation}
        J(\pi_E) - J(\pi) \geq \Omega(\Bar{\epsilon}T^2)
    \end{equation}
    \label{thm:mmdp-lb}
    \hyperref[pf:thm:mmdp-lb]{[Proof]}
\end{theorem}

Intuitively, a single mistake early on in an episode can put the learner in a different part of the state space than the expert, which can make the learned policy perform poorly. 
\begin{algorithm*}[t]
\begin{algorithmic}
\STATE {\bfseries Input:} Sequence of expert visitation distributions $\rho_{E}^1 \dots \rho_{E}^T$, Policy class $\Pi$, Reward class $\mathcal{F}_r$
\STATE {\bfseries Output:} Trained policy $\pi$
\STATE Set $\pi_0 \in \Pi$, $\mathcal{D} = \{\}$
\FOR{$i=1$ to $N$}
\STATE Set $\mathcal{D}_{i-1} = \{\}$
\FOR{$j=1$ to $M$}
\STATE Sample random time $t \sim \text{Unif}([0, T])$ and start state $s_t \sim \rho_{E}^t$.
\STATE Execute a random action $a_t \sim \text{Unif}(\mathcal{A})$ in $s_t$.
\STATE Follow $\pi_{i-1}$ until the end of the horizon.
\STATE $\mathcal{D}_{i-1} \gets \mathcal{D}_{i-1} \cup \{(s_t, a_t, t, s_{t+1:T}, a_{t+1:T})\}$
\ENDFOR
\STATE Let 
\begin{equation}
    L(\pi_{i-1}, f) = \mathbb{E}_{\xi \sim \rho_E}\left[\sum_{t=0}^T f(s_{t}, a_{t})\right] - \mathbb{E}_{\xi \sim \pi_{i-1}}\left[\sum_{t=0}^T f(s_{t}, a_{t})\right]  \label{eq:traj-alg}
\end{equation}
\STATE Optimize $f_{i-1} \gets \arg\max_{f \in \mathcal{F}_r} L(\pi_{i-1}, f)$. \hfill
\algcomment{for \texttt{NRMM(NR), optimize} $L(\text{Unif}(\pi_{1:i-1}), \cdot)$ instead}
\STATE $\mathcal{D} \gets \mathcal{D} \cup \{(s_t, a_t, \hat{Q}_t = \sum_{\tau =t}^T f_{i-1}(s_{\tau}, a_{\tau}) | \text{tuple} \in \mathcal{D}_{i-1}\}$
 \STATE \algcomment{Run any no-regret algorithm on $\mathcal{D}_{1:i-1}$ to produce new $\pi_i$, e.g. FTRL:}
\STATE Optimize \begin{equation}
    \pi_i \gets \argmax_{\pi \in \Pi} \mathbb{E}_{s \sim \mathcal{D}, a \sim \pi(s)}[\mathbb{E}[\hat{Q}_t|s_t = s, a_t = a]] + H(\pi).
\end{equation}
\ENDFOR
\STATE {\bfseries Return } $\pi_i$ with lowest validation error.
\end{algorithmic}
\caption{\texttt{NRMM(BR)} (No-Regret Moment Matching: Best Response Variant) \label{alg:filter}}
\end{algorithm*}
 In short, \texttt{MMDP} is able to find a sequence of policies in polynomial time that perform at most $\Bar{\epsilon} T^2$ worse than $\pi_E$. 

\noindent \textbf{\texttt{MMDP} vs. Behavioral Cloning.} A natural question at this point might be: \textit{what benefits does \texttt{MMDP} provide over a behavioral cloning baseline?} After all, behavioral cloning also produces policies that do no worse than $O(\epsilon T^2)$ compared to the expert and requires no environment interaction.

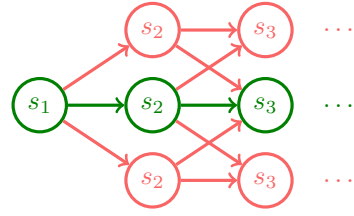
\begin{figure}[t]
    \centering
    \begin{tikzpicture}[scale=1, transform shape]
    \node (b) [draw, very thick, circle, color=expert] at (0, 1) {$s_1$};
    \node (d) [draw, very thick, circle, color=error] at (1.5, 2) {$s_2$};
    \node (e) [draw, very thick, circle, color=expert] at (1.5, 1) {$s_2$};
    \node (f) [draw, very thick, circle, color=error] at (1.5, 0) {$s_2$};
    \node (g) [draw, very thick, circle, color=error] at (3, 2) {$s_3$};
    \node (h) [draw, very thick, circle, color=expert] at (3, 1) {$s_3$};
    \node (i) [draw, very thick, circle, color=error] at (3, 0) {$s_3$};
    \node (j) [color=error] at (4, 2) {$\dots$};
    \node (k) [color=expert] at (4, 1) {$\dots$};
    \node (l) [color=error] at (4, 0) {$\dots$};

    \draw [->, very thick, color=error] (b) to (d);
    \draw [->, very thick, color=expert] (b) to (e);
    \draw [->, very thick, color=error] (b) to (f);

    \draw [->, very thick, color=error] (d) to (g);
    \draw [->, very thick, color=expert] (e) to (h);
    \draw [->, very thick, color=error] (f) to (i);
    
    \draw [->, very thick, color=error] (e) to (g);
    \draw [->, very thick, color=error] (e) to (i);

    \draw [->, very thick, color=error] (d) to (h);
    \draw [->, very thick, color=error] (f) to (h);

    \end{tikzpicture}
    \caption{\textsc{dante}: A three-row MDP where at each timestep, the learner can move up, move down, or stay in the same row. The expert always stays in the center row. The goal is to stay in the top two rows.\label{fig:dante}}
\end{figure}
Consider a simplified variant of \texttt{MMDP} in which one doesn't perform rollouts and instead solves a game with a single-timestep payoff at each iteration. This entirely decouples the iterations as we no longer account for the actions of the future policies we have already computed. In effect, this is what purely offline behavioral cloning is doing. 

The core issue with such an approach is that \textit{it prevents the learner from distinguishing between mistakes that compound over the horizon and those that don't}. Consider, for example, the MDP depicted in Figure~\ref{fig:dante} where the goal is to stay in the top two rows. Assume policies $\pi_{3:T}$ go straight but $\pi_2$ goes down w.p. $\epsilon T$. %
Now, let's think about what would happen if we used BC or \texttt{MMDP} to pick $\pi_1$. Behavioral cloning would pick a policy that always goes straight, as doing so perfectly matches expert actions. This would lead to a performance gap of $$J(\pi_E, r) - J(\{\pi_{BC}, \pi_{2:T}\}, r) = \epsilon T(T-1).$$ However, if we instead used \texttt{MMDP} to pick $\pi_1$, the rollouts with $\pi_{2:T}$ would reveal to the learner that it is better to go up on the first timestep so they still receive reward over the horizon, no matter what $\pi_2$ chooses. Thus, the learner would match expert performance, i.e. $$J(\pi_E, r) - J(\{\pi_{\texttt{MMDP}}, \pi_{2:T}\}, r) = 0.$$

So, while in the worst case, BC and \texttt{MMDP} might both perform poorly (e.g. if the learner falls off a cliff and is stuck for the rest of the episode), we would expect that for a wide set of practical problems, knowledge of future choices would enable better performance over the horizon.

\subsection{Method 2: No-Regret Moment Matching}

For tasks with long horizons, learning a sequence of policies may be significantly more burdensome than learning just one. We now present an algorithm that outputs a single, stationary policy. Our approach is based on the No-Regret Policy Iteration (NRPI) algorithm of \citet{ross2011}. Instead of solving a sequence of optimization problems backwards in time like PSDP, NRPI picks a time to sample from the roll-in distribution uniformly at random, takes a random action, and then follows the previous policy $\pi_{i-1}$ for the rest of the episode. This gives it sample estimates of $Q^{\pi_{i-1}}$ on states from the roll-in distribution. To have a no-regret property, NRPI performs (regularized) greedy policy improvement using the \textit{history} of such samples, i.e.
\begin{equation}
    \pi_{i} = \argmax_{\pi \in \Pi} \sum_{j=0}^{i-1} \mathbb{E}_{\substack{t \sim U[0, T], s \sim \eta^t}}[Q^{\pi_j}(s, \pi(s))] + H(\pi). \nonumber
\end{equation}
Notice that rather than solving a global exploration problem, NRPI only focuses on picking the best action on states from the roll-in distribution, avoiding the exponential interaction complexity lower bound. NRPI can be seen as an analog of PSDP for stationary policies \citep{ross2011}.

As with PSDP, NRPI requires a reward function. We therefore choose one adversarially for IRL. We outline the full procedure in Algorithm \ref{alg:filter}. Intuitively, this algorithm is performing \textit{primal} moment-matching with the learner's start state distribution being the expert's stationary distribution (i.e. Algorithm \ref{alg:primal-irl} or GAIL with expert resets). For space reasons, we postpone the dual algorithm to Appendix \ref{app:dual}.

Let $L(\pi, \mathcal{D}_i) = \mathbb{E}_{s \sim \rho_E, a \sim \pi(s)}[\mathbb{E}_{\mathcal{D}_i}[\hat{Q}_t \vert s_t = s, a_t = a]]$ denote the cost-sensitive classification loss of policy $\pi$ over dataset $\mathcal{D}_i$.
We use the following regret measure in our analysis:
\begin{equation}
    \epsilon_i =  L(\pi^*, \mathcal{D}_i) - L(\pi_i, \mathcal{D}_i), \label{eq:nrmm-eps}
\end{equation}
where $\pi^* = \argmax_{\pi \in \Pi} \sum_i^N L(\pi, \mathcal{D}_i)$.
Like \texttt{MMDP}, \texttt{NRMM} has polynomial time iterations.
\begin{lemma}{\textbf{\texttt{NRMM} Sample Complexity:}} Let $\epsilon > 0$. At iteration $i$, \texttt{NRMM} requires at most $$O\left(\log\left(\frac{|\Pi||\mathcal{F}_r|}{\delta}\right) \frac{T^3|\mathcal{A}|^2}{\epsilon^2}\right)$$ interactions with the environment to, w.p. $\geq 1 - \delta$, produce a policy $\pi_i$ with instantaneous regret $\epsilon_i \leq \epsilon$ (Eq. \ref{eq:nrmm-eps}). \hyperref[pf:thm:nrmm-sc]{[Proof]}
\label{thm:nrmm-sc}
\end{lemma}

However, unlike \texttt{MMDP} which always has $T$ outer-loop iterations, \texttt{NRMM} must be run until the average training error drops below some threshold on $\Bar{\epsilon}$. While the particular number of iterations $N$ is a problem-specific quantity, the fact that the policy is selected by a no-regret algorithm tells us that, by definition, \begin{equation}
    \lim_{N \to \infty} \frac{1}{N}\sum_{i=1}^N \epsilon_i = 0.
\end{equation}
Thus, regardless of the desired $\Bar{\epsilon}$, the outer loop will eventually terminate, with $\Bar{\epsilon} \propto \frac{1}{\sqrt{N}}$ or $\Bar{\epsilon} \propto \frac{\log (N)}{N}$ for a wide set of problems \citep{HazanTextbook}, giving us poly-time bounds.
There exist two variations of \texttt{NRMM}: one in which the adversary plays a best-response (i.e. differentiating between the current policy and expert demos -- labeled as \texttt{NRMM(BR)}) and another in which the adversary follows a no-regret strategy (i.e. differentiating between replay buffer $\mathcal{D}$ and expert demos -- labeled as \texttt{NRMM(NR)}). Both share similar policy performance guarantees (Appendix \ref{app:proofs}).

\begin{theorem}{\textbf{\texttt{NRMM(BR)} Upper Bound:}} 
    Let $\pi_1, \dots, \pi_N$ denote the sequence policies computed by \texttt{NRMM(BR)} and $\Bar{\epsilon} = \frac{1}{N} \sum_{i=1}^N \epsilon_i$ their average regret (Eq. \ref{eq:nrmm-eps}). Then, $\exists \pi \in \{\pi_1, \dots, \pi_N\}$ s.t.
    \begin{equation}
        J(\pi_E) - J(\pi) \leq \Bar{\epsilon}T^2
    \end{equation}
    \hyperref[pf:thm:filter-br]{[Proof]}
    \label{thm:filter-br}
\end{theorem}
When we use a no-regret algorithm to pick $f_i$ rather than a best response, we need to consider the instantaneous regrets of said algorithm. Let $f^* = \argmax_{f \in \mathcal{F}_r} \sum_{i=1}^N L(\pi_i, f)$ and \begin{equation}
    \delta_i = L(\pi_i, f^*) - L(\pi_i, f_i), \label{eq:nrmm-delta}
\end{equation}
where $L$ is as defined in Eq. \eqref{eq:traj-alg}. We can now give a performance guarantee as a function of $\epsilon_i$ and $\delta_i$.
\begin{theorem}{\textbf{\texttt{NRMM(NR)} Upper Bound:}} 
    Let $\pi_1, \dots, \pi_N$ and $f_1, \dots, f_N$ denote the sequence policies and rewards computed by \texttt{NRMM(NR)} and $\Bar{\epsilon} = \frac{1}{N} \sum_{i=1}^N \epsilon_i$, $\Bar{\delta} = \frac{1}{N} \sum_{i=1}^N \delta_i$ their respective average regrets (Eqs. \ref{eq:nrmm-eps}, \ref{eq:nrmm-delta}). Then, the uniform mixture over policies $\Bar{\pi}$ satisfies
    \begin{equation}
        J(\pi_E) - J(\Bar{\pi}) \leq (\Bar{\epsilon} + \Bar{\delta})T^2
    \end{equation}
    \label{thm:filter-nr}
    \hyperref[pf:thm:filter-nr]{[Proof]}
\end{theorem}
These bounds are tight, via a similar construction to before.
\begin{theorem}{\textbf{\texttt{NRMM} Lower Bound:}} 
    There exists an MDP, $\pi_E$ and $\pi$ with average training error $\Bar{\epsilon}$ such that
    \begin{equation}
        J(\pi_E) - J(\pi) \geq \Omega(\Bar{\epsilon}T^2)
    \end{equation}
    \label{thm:filter-lb}
    \hyperref[pf:thm:filter-lb]{[Proof]}
\end{theorem}
As \texttt{NRMM} also performs rollouts in the environment, our discussion on why \texttt{MMDP} is preferable to behavioral cloning also applies to \texttt{NRMM}. 

We now highlight a nuance related to \texttt{NRMM}.

\noindent\textbf{Discriminator Training.} We prove that the standard trajectory-level discriminator training usually performed in IRL (i.e. Eq. \ref{eq:traj-alg} in Algorithm \ref{alg:filter}) is lower variance than the suffix-level discriminator training one might think to perform based on the samples in replay buffer $\mathcal{D}$. We prove this point more formally in Appendix \ref{app:proofs}. In practice, we find that trajectory-level discriminator training works better than suffix-level discriminator training and therefore utilize it in all of our implementations.

\subsection{Dual Algorithms}
A natural question upon reading the preceding sections is whether \textit{dual} algorithms can leverage expert resets to speed up policy search. Practically, these algorithms would run PSDP or NRPI in their \textit{inner loop} with a reward function chosen via a no-regret algorithm in their \textit{outer loop}. Indeed these dual algorithms also work, but for a subtle reason. \footnote{In an earlier draft of this paper, we arrived at the incorrect answer of \textit{no} to the preceding question. We contain multitudes.} 

Recall that NRPI and PSDP only compete with policies that have similar visitation distributions to the expert \citep{bagnell2003policy, ross2011}. This is fine when selecting policies in the outer loop as the expert policy is an equilibrium strategy. However, the story is less clear when policy search is the inner loop. This is because the expert policy might be quite far from the optimal policy for the adversarially chosen reward. Thus, if we use NRPI/PSDP as our policy search method, the learner may struggle to find the \textit{best response} needed for equilibrium computation. However, we prove that we're still able to guarantee we learn strong policies on average over iterations via both dual algorithms. Intuitively, one can guarantee doing as well as $\pi_E$ under all reward functions simply by doing as well as $\pi_E$ at each iteration of a no-regret reward selection algorithm. Note that this does not require finding the
truly optimal policy for each adversarially selected reward function. Put differently, an \textit{expert-competitive response} suffices if a \textit{best response} is not possible.
We give the full psuedocode for the dual algorithms and prove similar performance bounds to those in the preceding section in Appendix \ref{app:dual}.

\section{Getting the Best of Both Worlds}

In the preceding section, we derived two algorithms, \texttt{MMDP} and \texttt{NRMM}, which can compute policies that match expert behavior in polynomial time. However, in the worst case, both can produce policies that suffer from a quadratic compounding of errors with respect to the horizon. Traditional IRL approaches have complimentary strengths: they can suffer from exponential computation complexity but produce policies with a performance gap linear in the horizon. This begs the question: \textit{can we get the best of both worlds?}

Consider a variation of \texttt{NRMM} where, with probability $\alpha$, we perform an expert reset, otherwise performing a standard rollout (i.e. $s_t \sim \rho_{\pi_{i-1}}^t$). By setting $\alpha = 1$, we unsurprisingly recover \texttt{NRMM}. However, if we set $\alpha = 0$, the per-round loss that is passed to the learner becomes
\begin{equation}
    L(\pi, \mathcal{D}_i) = \mathbb{E}_{s \sim \rho_i, a \sim \pi(s)}[\mathbb{E}_{\mathcal{D}_i}[\hat{Q}_t \vert s_t = s, a_t = a]],
\end{equation}
This is strikingly similar to the standard approximate policy improvement procedure \citep{sutton2018reinforcement} with an adversarially chosen reward. Recall that in \texttt{NRMM}, we select our discriminator $f$ as in primal IRL (Algorithm \ref{alg:primal-irl}). Put together, setting $\alpha = 0$ is effectively using an off-policy RL algorithm in the policy optimization component of Algorithm \ref{alg:primal-irl}. One might therefore reasonably expect such an approach to inherit the exponential complexity and linear-in-the-horizon performance gap of standard IRL.

It is natural to consider annealing between these extremes by decaying $\alpha$ from $1$ to $0$ over outer-loop iterations. Intuitively, this allows the learner to quickly find a policy with quadratic errors before refining it to a policy with error linear in the horizon. Even more simply, one can interpolate with a fixed $\alpha = 0.5$ probability, reducing the exploration burden on the learner while mitigating compounding errors. We term such annealed / interpolated approaches \texttt{FILTER}: Fast Inverted Loop Training via Expert Resets. Defining $\Bar{\epsilon}$ as in Eq. \eqref{eq:nrmm-eps} and $\Bar{\epsilon}_{RL}$ as
\begin{equation}
    \Bar{\epsilon}_{\text{RL}} =  \frac{1}{NT} \sum_i^N (\max_{f_i \in \mathcal{F}_r} J(\pi_E, f_i) - J(\pi_i, f_i)), \label{eq:rl-err}
\end{equation}
(i.e. the errors on the expert and start state distributions) we can derive a performance bound for \texttt{FILTER} by taking the minimum over the \texttt{NRMM} and IRL bounds.
\begin{corollary}{\textbf{\texttt{FILTER} Upper Bound:}}  Consider a set of policies $\{\pi_1, \dots, \pi_N\}$ with errors $\Bar{\epsilon}$ and $\Bar{\epsilon}_{\text{RL}}$ (Eqs. \eqref{eq:nrmm-eps}, \eqref{eq:rl-err}). Then, we have that ${\exists \pi \in \{\pi_1, \dots, \pi_N\}}$ s.t.
\begin{equation}
     J(\pi_E) - J(\pi) \leq \min \left ( \Bar{\epsilon} T^2, \Bar{\epsilon}_{\text{RL}}T \right ).
\end{equation}
\end{corollary}
The expert reset probability $\alpha$ controls the trade-off or schedule of minimizing $\Bar{\epsilon}$ ($\alpha \approx 1$) versus $\Bar{\epsilon}_{\text{RL}}$ ($\alpha \approx 0$). Intuitively, \texttt{FILTER} inherits the transferability of the reward function across problems of IRL, has better robustness to inaccuracy in $\rho_E$ and compounding errors than \texttt{NRMM}, and is better able to handle recoverable situations than behavioral cloning.

Unfortunately, it is difficult to prove more about \texttt{FILTER}. This is because a learner's performance on a mixture of two distributions doesn't easily translate to a bound on their performance on either. Traditional approaches to deriving such a bound (e.g. as function of the $\mathcal{H}\Delta\mathcal{H}$ divergence \citep{ben2010theory}) produce vacuous bounds when applied to flexible hypothesis classes like neural networks. Similar difficulties have been encountered by others in the IRL community without resolution \citep{chang2015learning}. 

\section{Experiments}

We conduct experiments with the PyBullet Suite \cite{coumans2019}. We train experts using RL and then present all learners with 25 expert demonstrations to remove small-data concerns. As a simple behavioral cloning baseline matches expert performance under these conditions \citep{swamy2021moments}, we harden the problem by introducing randomization: with probability $p_{tremble}$, a random action gets executed in the environment rather than the one the policy chose. Our expert data is free from these corruptions. We also conduct experiments on the \texttt{antmaze-large} tasks from \citet{fu2020d4rl}, but with $p_{tremble}=0$.

\begin{figure*}[t]
    \centering
    \includegraphics[width=0.26\textwidth]{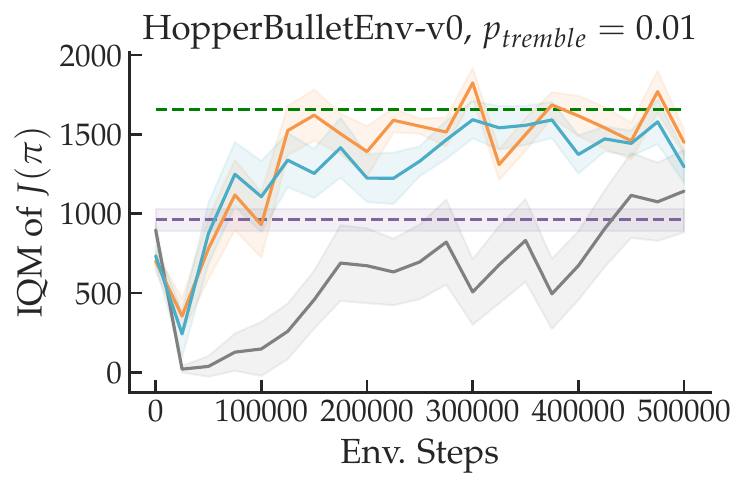}
    \includegraphics[width=0.26\textwidth]{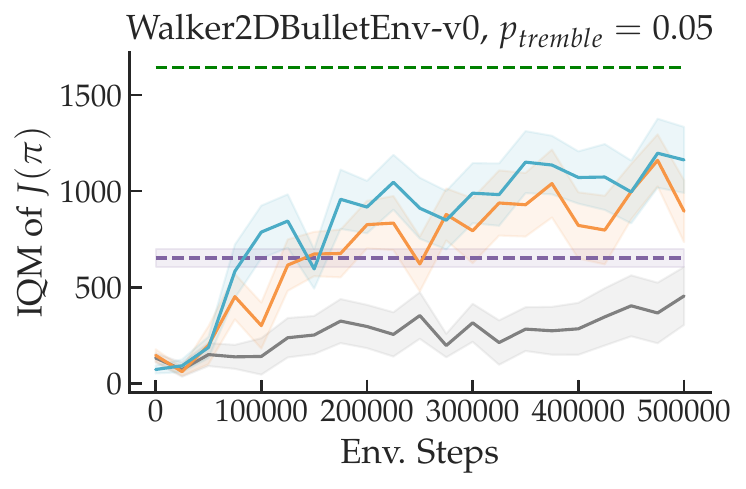}
    \includegraphics[width=0.26\textwidth]{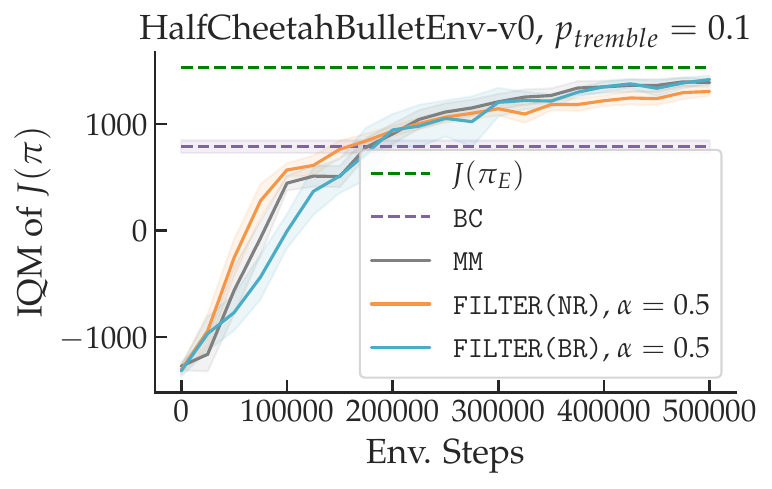}
    \includegraphics[width=0.25\textwidth]{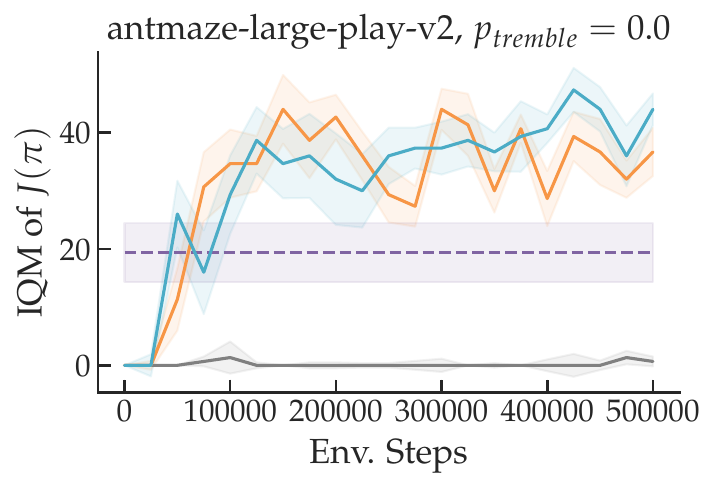}
    \includegraphics[width=0.25\textwidth]{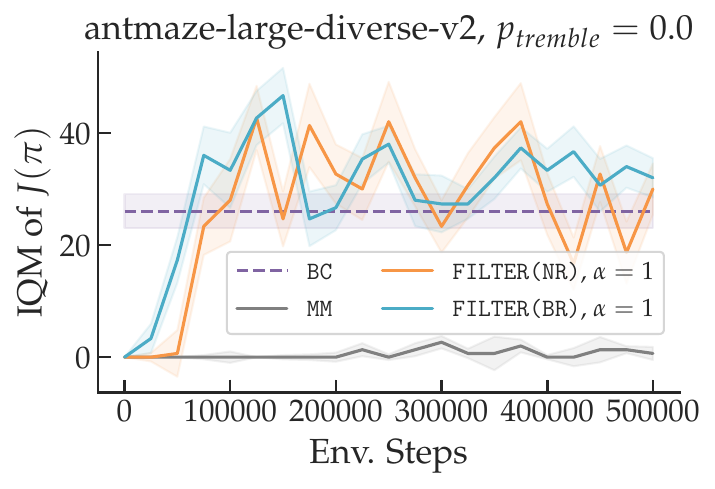}
    \caption{We see that both \texttt{FILTER(BR)} and \texttt{FILTER(NR)} out-performs standard \texttt{MM} and \texttt{BC} on 4 out of the 5 environments considered. Standard errors are computed across 10 seeds.
  }
    \label{fig:exp}
\end{figure*}

 We compare 4 algorithms: \texttt{FILTER(BR)}, \texttt{FILTER(NR)}, \texttt{MM} (i.e. Algorithm \ref{alg:primal-irl}, or, equivalently, \texttt{FILTER(NR)} with $\alpha=0$), and BC. \footnote{In practice, rather than perform policy improvement on just the first state from the learner suffix, we instead perform policy improvement
on states from the entire suffix (i.e. standard $Q$-learning). While in the worst case, this means that our performance bounds could degrade by a factor of $T$, in practice the benefits of leveraging the entire
suffix often outweighs the potential cost.} See Appendix \ref{app:exps} for details. We do not implement \texttt{MMDP} as these tasks can all last for $T=1000$ timesteps. We plot the performance of the policy as a function of the number of environment interactions used for policy optimization. \footnote{In some implementations of algorithms like GAIL, trajectories from the policy's replay buffer are used for training the discriminator rather than trajectories sampled post-policy-update. For \texttt{FILTER}, as we may only observe suffixes when $\alpha > 0$, we need to separately sample whole trajectories post-policy-update. To make the comparison fair, we do this for \texttt{MM} as well.} As recommended by \citet{agarwal2021deep}, we plot a robust statistic (i.e.\ the interquartile mean). Standard errors are computed across 10 runs.

 For our baseline moment-matching algorithm, we use a significantly improved version of GAIL \citep{ho2016generative}. Specifically, we switch from the Jensen-Shannon divergence to an integral probability metric (as recommended by \citet{swamy2021moments}), use the more efficient Soft Actor Critic \citep{haarnoja2018soft} or TD3+BC \citep{fujimoto2021minimalist} as our policy optimizers, add a gradient penalty to the discriminator \citep{gulrajani2017improved}, and use Optimistic Mirror Descent \citep{daskalakis2017training} to optimize both players for fast and last iterate convergence. See the appendix of \citet{swamy2022minimax} for an ablation of these changes. Taken together, these changes make our baseline a strong point of comparison, over which improvement is non-trivial.

In Figure \ref{fig:exp}, we see that \texttt{FILTER(BR)} and \texttt{FILTER(NR)} perform comparably and are significantly faster at finding strong policies than \texttt{MM} on 4/5 environments. We would recommend trying both variants when applying the algorithm in practice. To the best of our knowledge, the performance of \texttt{FILTER} on both variants of \texttt{antmaze} is the highest performance ever achieved by an algorithm that doesn't use any reward information. \footnote{We note that the performance we report for behavioral cloning on these environments is significantly higher than what is usually reported in the literature -- see Appendix \ref{app:exps} for details.}

It is also interesting to consider the difference in results between the environments we consider. In the Bullet locomotion environments, we found that $\alpha=0.5$ worked better than $\alpha=1$. We hypothesize that this is because the learner is able to learn to connect their initial state to sampled expert states more easily. For locomotion tasks, this might correspond to learning to accelerate before matching the expert’s gait. We tried a more complex annealing strategy but found that it did not outperform a fixed $\alpha=0.5$. However, we believe that for other problems, the annealing strategy could perform better than a fixed $\alpha$. For the AntMaze environments, we found that $\alpha=1$ worked better than lower values. We hypothesize that this is because of the difficulty of exploration in a maze, for which expert resets can help a lot. In general, we would recommend that the harder exploration is in a problem, the higher $\alpha$ should be set. We release the code we used for all of our experiments at \textbf{\texttt{\url{https://github.com/gkswamy98/fast_irl}}}. Of particular interest are the gym wrappers, which should be easily transferable to other algorithms / implementations.

\section{Discussion}
In summary, we provide multiple algorithms for more sample efficient inverse reinforcement learning, both in theory and practice. Our key insight is speeding up policy optimization via resetting the learner to states from expert demonstrations. We emphasize that due to the reduction-based analysis we perform, one could apply this technique to an arbitrary inverse reinforcement learning algorithm and not just the GAIL-like approach we use for experiments. 

One interesting avenue for future work is developing an algorithm with stronger guarantees in the interpolated case -- for example, one could imagine training two discriminators (one on trajectories from each start state distribution) and using the more pessimistic during learning. 

Another is to address an assumption fundamental to our approach: the ability to reset the learner to an arbitrary initial state. While this is possible in many (if not most) simulators, it is not clear how to do this in the real world or when one only has the "trace" model of access (i.e. resets only to a fixed initial state distribution).

Lastly, one could also further investigate where sub-optimal data could be used in our procedure. For example, one could mix it with the expert data and use this mixture distribution for resets if only a limited number of demonstrations are available. As long as we still use the expert data for reward selection, we conjecture that similar guarantees to the ones we prove above would hold.

\section{Acknowledgments}
We thank Vasilis Syrgkanis and Juntao Ren for pointing out several typos. ZSW is supported in part by the NSF FAI Award \#1939606, a Google Faculty Research Award, a J.P. Morgan Faculty Award, a Facebook Research Award, an Okawa Foundation Research Grant, and a Mozilla Research Grant. GS is supported computationally by a GPU award from NVIDIA and emotionally by his family and friends. %

\clearpage

\bibliographystyle{unsrtnat}
\bibliography{references}  %

\appendix
\onecolumn
\section{Proofs}
\localtableofcontents
\label{app:proofs}
\subsection{Sample Complexity Lemma}
    We follow standard techniques and include the proof here mostly for completeness.
\begin{lemma}
    Consider $C$ deterministic functions of a random variable, each with range $R$. If we draw
    \begin{equation}
    m \geq O\left(\log\left(\frac{C}{\delta}\right) \frac{R^2}{\epsilon^2}\right)
\end{equation}
    samples, we have that with probability $\geq 1 - \delta$, we will be able to estimate all $C$ population means within $\epsilon$ absolute error.
    \label{lem:sc}
\end{lemma}
\begin{proof}
Consider a bounded random variable $X$ with range $R$. A standard Hoeffding bound tells us that
\begin{equation}
    P(|\frac{1}{m}\sum_{i=0}^m X_i - \mathbb{E}[X]| \geq \epsilon) \leq 2 \exp(\frac{-2m \epsilon^2}{R^2}).
\end{equation}
If we have $C$ such variables and want to be within $\epsilon$ of the population mean uniformly, a union bound tells us that we will do so with probability at least
\begin{equation}
    1 - 2C \exp(\frac{-2m \epsilon^2}{R^2}).
\end{equation}
If we want to satisfy this condition with probability at least $1 - \delta$, simple algebra tells us that we must draw
\begin{equation}
    m \geq O\left(\log\left(\frac{C}{\delta}\right) \frac{R^2}{\epsilon^2}\right)
\end{equation}
samples.
\end{proof}

\subsection{Proof of Theorem \ref{thm:irl-exp}}
\label{pf:thm:irl-exp}
This construction is essentially a slight generalization of that of \citet{kakade2003sample} to the case with multiple reward functions.
\begin{proof}
Consider a tree-structured MDP with branching factor $|\mathcal{A}|$ and deterministic dynamics. The expert always takes the left-most action and therefore always ends up at the left-most node. Let $\mathcal{F}_r$ be the set of sparse reward functions that are $1$ at a single leaf node and 0 everywhere else. Let $\Pi$ be the full set of deterministic policies (i.e. paths to a leaf node). Note that $|\mathcal{F}_r| = |\Pi| = |\mathcal{A}|^T$. Also note that $V_{max} = 1$ and that only one $\pi \in \Pi$ achieves nonzero reward under the true reward function, so one needs to find $\pi_E$ to satisfy the condition in the theorem statement.

Let us first analyze the dual version of IRL (Algorithm \ref{alg:dual-irl}). At each iteration, the policy player solves a fresh RL problem with $r = -f \in \mathcal{F}_r$. As all $f \in \mathcal{F}_r$ are sparse, the learner needs to visit all nodes in the tree to find which one provides reward. As $|\mathcal{S}| \geq \Omega(|\mathcal{A}|^T)$, this must take at least $\Omega(|\mathcal{A}|^T)$ interactions with the environment.

We now analyze the primal version of IRL (Algorithm \ref{alg:primal-irl}). While for $i > 1$ there could now exist multiple leaf nodes with reward under aggregate reward function $r = \frac{-1}{i} \sum_{j=1}^i f_j$, the learner has to contend with the fact that rewards corresponding to certain leaf nodes could have been chosen more than once by the adversary, giving  reward $> \frac{1}{i}$. Thus, the learner still needs to visit all leaf nodes, which again takes $\Omega(|\mathcal{A}|^T)$ interactions with the environment

\end{proof}

\subsection{Proof of Lemma \ref{thm:mmdp-sc}}
\label{pf:thm:mmdp-sc}
\begin{proof}
        At the $t$th iteration of \texttt{MMDP}, we are solving a two-player zero-sum game over strategy spaces $\Pi$ and $\mathcal{F}_r$ with payoff given by Equation \ref{eq:mmdp-payoff}. All interaction with the environment happens during the collection of $\mathcal{D}_t$ so we analyze how many iterations $M$ we must perform to estimate the payoff matrix within $\epsilon_t$ uniformly w.p $\geq 1- \delta$.
        
        First, note that there are $C = |\Pi||\mathcal{F}_r|$ elements in the matrix. Second, observe that each element of the matrix is within $[-T, T]$ before the $\frac{1}{T}$ normalization. Third, notice that the outer expectation in the first half of Equation \ref{eq:mmdp-payoff} is taken with respect to the policy while we collect data by sampling $a_t$ uniformly at random. Thus, to estimate this term, we use importance weighting between the learner and uniform policies. The maximum value of such a weight (corresponding to a deterministic learner policy) is $\frac{1}{|\mathcal{A}|}$. Thus, the overall scale of the random variable corresponding to each element of the payoff matrix is $R = T|\mathcal{A}|$. Now, applying Lemma \ref{lem:sc}, we see that we need
        \begin{equation}
        M \geq O\left(\log\left(\frac{|\Pi||\mathcal{F}_r|}{\delta}\right) \frac{(T|\mathcal{A}|)^2}{\epsilon^2}\right)
        \end{equation}
        trajectories, each of which could take $O(T)$ interactions with the environment, giving us an overall interaction complexity bound of
        \begin{equation}
        O\left(\log\left(\frac{|\Pi||\mathcal{F}_r|}{\delta}\right) \frac{T^3|\mathcal{A}|^2}{\epsilon^2}\right) \leq \text{poly}(T, |\mathcal{A}|, \frac{1}{\epsilon}, \log(\frac{1}{\delta}), \log(|\Pi|), \log(|\mathcal{F}_r|).
        \end{equation}

        Observe that with this many samples, we are able to estimate all elements of the $\pi$ and $f$ payoff matrix to within $\epsilon$ w.p. $\geq 1 - \delta$. Thus, the error we could accumulate by optimizing over the empirical rather than the population payoff matrix is bounded by $\epsilon \leq \epsilon_t$. 
\end{proof}

\subsection{Proof of Theorem \ref{thm:mmdp}}
\label{pf:thm:mmdp}
\begin{proof}
    Let $Q_f^{\pi^t \dots \pi^T}(s, a)$ denote the expected cumulative value of $f$ on trajectories generated by rolling out $\pi^t$ through $\pi^T$ starting from $(s, a)$. Then, via the Performance Difference Lemma \citep{kakade2002approximately},
    \begin{align}
        J(\pi_E) - J(\pi) &= \sum_{t=0}^T \mathbb{E}_{\xi \sim \pi_E}[Q_r^{\pi^{t+1} \dots \pi^T}(s_t, a_t) - \mathbb{E}_{a \sim \pi^t}[Q_r^{\pi^{t+1} \dots \pi^T}(s_t, a)]] \\
        &\leq \sum_{t=0}^T \sup_{f_t \in \mathcal{F}_r} \mathbb{E}_{\xi \sim \pi_E}[Q_{f_t}^{\pi^{t+1} \dots \pi^T}(s_t, a_t) - \mathbb{E}_{a \sim \pi^t}[Q_{f_t}^{\pi^{t+1} \dots \pi^T}(s_t, a)]] \\
        &\leq \sum_{t=0}^T T \epsilon_t = \Bar{\epsilon}T^2.
    \end{align}
\end{proof}

\subsection{Proof of Theorem \ref{thm:mmdp-lb}}
\label{pf:thm:mmdp-lb}
\begin{proof}
    We consider the \textsc{Cliff} MDP of \citet{swamy2021moments}, which we reproduce here for convenience.
    \begin{figure}[H]
        \centering
        \begin{tikzpicture}[scale=1, transform shape]
    \node (a) [draw, very thick, circle] at (0.0, 0) {$s_0$};
    \node (b) [draw, very thick, circle, color=expert] at (1.5, 0) {$s_1$};
    \node (c) [draw, very thick, circle, color=expert]  at (3, 0) {$s_2$};
    \node (d) [circle]  at (4.5, 0) {$\ldots$};
    \node (e) [draw, very thick, circle, color=error]  at (1.5, -1.5) {$s_x$};
    \path[->,] (a) to node[midway, above, color=expert] {$a_1$} (b);
    \draw [->, very thick, color=expert] (a) to (b);
    \path[->] (b) to node[midway, above, color=expert] {$a_1$} (c);
    \draw [->, very thick, color=expert] (b) to (c);
    \path[->] (c) to node[midway, above, color=expert] {$a_1$} (d);
    \draw [->, very thick, color=expert] (c) to (d);
    \path[->] (a) to node[midway, left, color=error] {$a_2$} (e);
    \draw [->, very thick, color=error] (a) to (e);
    \path[->] (b) to node[midway, left, color=error] {$a_2$} (e);
    \draw [->, very thick, color=error] (b) to (e);
    \path[->] (c) to node[midway, left, color=error] {$a_2$} (e);
    \draw [->, very thick, color=error] (c) to (e);

    \node [circle, minimum size=0.5cm](g) at ([{shift=(270:0.4)}]e){};
    \coordinate (h) at (intersection 2 of e and g);
    \coordinate (i) at (intersection 1 of e and g);
    \tikzAngleOfLine(g)(i){\AngleStart}
    \tikzAngleOfLine(g)(h){\AngleEnd}
    \draw[very thick,->]%
    let \p1 = ($ (g) - (i) $), \n2 = {veclen(\x1,\y1)}
    in
        (g) ++(270:0.5) node{$a_1$}
        (i) arc (\AngleStart-360:\AngleEnd:\n2);
    \end{tikzpicture}
        \label{fig:cliff_mdp}
    \end{figure}
    Assume the expert always takes $a_1$ and $r(s, a) = -\mathds{1}_{s_x} -\mathds{1}_{a_2}$. Thus, $J(\pi_E, r) = 0$. Assume that $\mathcal{F}_r = \{r\}$.
    
    Let $\pi_a$ be the policy that takes $a_2$ with prob. $\epsilon T$ in $s_0$ and $a_1$ otherwise. Let $\pi_b$ be the policy that always takes $a_1$. Let $\pi = \{\pi_a, \pi_b, \dots\}$ be the sequence of policies returned by \texttt{MMDP}. 
    
    For the first $T-1$ steps of the algorithm, $\epsilon_t = 0$ as the learner plays $\pi_b$. On the last step of the algorithm, the learner picks a policy $\pi_a$ which makes mistakes for the rest of the horizon w.p. $\epsilon T$, giving it a moment matching error of $\epsilon_1 = \epsilon T$. Thus, overall, $\pi$ has average moment-matching error $\Bar{\epsilon} = \frac{1}{T}(\epsilon T + \sum_{t=2}^T 0) = \epsilon$. However, on rollouts, the learner would have an $\epsilon T$ chance of paying a cost of 1 for the rest of the horizon, leading to a lower bound of $J(\pi_E, r) - J(\pi, r) = \epsilon T^2 \geq \Omega(\epsilon T^2)$.
\end{proof}

\subsection{Proof of Lemma \ref{thm:nrmm-sc}}
\label{pf:thm:nrmm-sc}
\begin{proof}
    We proceed similarly to the proof of Theorem \ref{thm:mmdp-sc}. All interaction with the environment happens during the $M$ interactions with the environment. As before, we are estimating a payoff matrix with $C = |\Pi||\mathcal{F}_r|$ elements within $\epsilon_i$ uniformly w.p $\geq 1- \delta$. Each element has scale $R = T|\mathcal{A}|$. Applying Lemma \ref{lem:sc}, we see that we need
        \begin{equation}
        M \geq O\left(\log\left(\frac{|\Pi||\mathcal{F}_r|}{\delta}\right) \frac{(T|\mathcal{A}|)^2}{\epsilon_i^2}\right)
        \end{equation}
        trajectories, each of which could take $O(T)$ interactions with the environment, giving us an overall interaction complexity bound of
        \begin{equation}
        O\left(\log\left(\frac{|\Pi||\mathcal{F}_r|}{\delta}\right) \frac{T^3|\mathcal{A}|^2}{\epsilon^2}\right) \leq \text{poly}(T, |\mathcal{A}|, \frac{1}{\epsilon}, \log(\frac{1}{\delta}), \log(|\Pi|), \log(|\mathcal{F}_r|).
        \end{equation}

        Observe that with this many samples, we are able to estimate all elements of the $\pi$ and $f$ payoff matrix to within $\epsilon$ w.p. $\geq 1 - \delta$. Thus, the error we could accumulate by optimizing over the empirical rather than the population payoff matrix is bounded by $\epsilon \leq \epsilon_i$. To complete the proof, observe that this bounds the optimization error (i.e. difference in value between $\pi_i$ and the per-round best response policy when plugged into Eq. \ref{eq:nrmm-eps}) which upper bounds the instantaneous regret (i.e. difference in value between $\pi_i$ and the best-in-hindsight policy when plugged into Eq. \ref{eq:nrmm-eps}).
    
\end{proof}

\subsection{Proof of Theorem \ref{thm:filter-br}}
\label{pf:thm:filter-br}
\begin{proof}
First, we note that
    \begin{align}
        J(\pi_E) - J(\pi) &= \sum_{t=1}^T \mathbb{E}_{\xi \sim \pi_E}[Q_r^{\pi}(s_t, a_t) - \mathbb{E}_{a \sim \pi}[Q_r^{\pi}(s_t, a)]] \\
        &= \sum_{t=1}^T \mathbb{E}_{s, a \sim \rho_t^E}[ \mathbb{E}_{\xi \sim \pi|s, a}[\sum_{\tau = t}^T r(s_{\tau}, a_{\tau})]] - \mathbb{E}_{a' \sim \pi(s)}[\mathbb{E}_{\xi \sim \pi|s, a'}[\sum_{\tau = t}^T r(s_{\tau}, a_{\tau})]]] \label{eq:suffix}\\
        &= \sum_{t=1}^T \mathbb{E}_{s_t, a_t \sim \rho_t^{E}}[r(s_t, a_t)] - \mathbb{E}_{s_t, a_t \sim \rho_t^{\pi}}[r(s_t, a_t)]. \label{eq:traj}
    \end{align}

The first equality is via the PDL, the second via the definition of a $Q$ function, and the third by the definition of $J$. Next, we set
\begin{equation}
    f_i^* = \arg\max_{f \in \mathcal{F}_r} \sum_{t=1}^T \mathbb{E}_{s_t, a_t \sim \rho_t^{\pi_i}}[f(s_t, a_t)] - \mathbb{E}_{s_t, a_t \sim \rho_t^{E}}[f(s_t, a_t)]
\end{equation}
and define
\begin{equation}
    L_i(\pi, f) = \frac{1}{T} \sum_{t=0}^T \mathbb{E}_{s \sim \rho_t^{E}}[\mathbb{E}_{a \sim \pi(s)}[\mathbb{E}_{\xi \sim \pi_i|s, a}[\sum_{\tau = t}^T f(s_{\tau}, a_{\tau})]]].
\end{equation}
Note the iteration-indexed "roll-out" policy. We use this sequence of loss functions to define a regret measure,
\begin{equation}
    \Bar{\epsilon} = \frac{1}{NT} \sum_{i=1}^N L_i(\pi_i, f_i^*) - \min_{\pi \in \Pi}\frac{1}{NT} \sum_{i=1}^N L_i(\pi, f_i^*) \in [-1, 1],
\end{equation}
and $\Bar{\pi}$ to denote the uniform mixture over policy iterates. Now, by our earlier equalities,
\begin{align}
    J(\pi_E) - J(\Bar{\pi})  &= \frac{1}{N} \sum_{i=1}^N J(\pi_E) - J(\pi_i) \\
    &= \frac{1}{N} \sum_{i=1}^N \sum_{t=1}^T \mathbb{E}_{s, a \sim \rho_E^t}[\mathbb{E}_{\xi \sim \pi_i|s, a}[\sum_{\tau = t}^T r(s_{\tau}, a_{\tau})]] - \mathbb{E}_{a' \sim \pi_i(s)}[\mathbb{E}_{\xi \sim \pi_i|s, a'}[\sum_{\tau = t}^T r(s_{\tau}, a_{\tau})]]] \\
    &\leq \sup_{f \in \mathcal{F}_r} \frac{1}{N} \sum_{i=1}^N \sum_{t=1}^T \mathbb{E}_{s, a \sim \rho_E^t}[\mathbb{E}_{\xi \sim \pi_i|s, a}[\sum_{\tau = t}^T f(s_{\tau}, a_{\tau})]] - \mathbb{E}_{a' \sim \pi_i(s)}[\mathbb{E}_{\xi \sim \pi_i|s, a'}[\sum_{\tau = t}^T f(s_{\tau}, a_{\tau})]]] \\
    &\leq \frac{1}{N} \sum_{i=1}^N \sup_{f_i \in \mathcal{F}_r} \sum_{t=1}^T \mathbb{E}_{s, a \sim \rho_E^t}[\mathbb{E}_{\xi \sim \pi_i|s, a}[\sum_{\tau = t}^T f_i(s_{\tau}, a_{\tau})]] - \mathbb{E}_{a' \sim \pi_i(s)}[\mathbb{E}_{\xi \sim \pi_i|s, a'}[\sum_{\tau = t}^T f_i(s_{\tau}, a_{\tau})]]] \\
    &= \frac{1}{N} \sum_{i=1}^N \sup_{f_i \in \mathcal{F}_r} T (L_i(\pi_i, f_i) - L_i(\pi_E, f_i)) \\
    &= \frac{1}{N} \sum_{i=1}^N T (L_i(\pi_i, f_i^*) - L_i(\pi_E, f_i^*)).
\end{align}
Set $\pi^* = \arg\min_{\pi \in \Pi} \sum_{i=1}^N L_i(\pi, f_i^*)$. Continuing,
\begin{align}
    J(\pi_E) - J(\Bar{\pi})
    &\leq \frac{1}{N} \sum_{i=1}^N T (L_i(\pi_i, f_i^*) - L_i(\pi^*, f_i^*)) \\
    &= \Bar{\epsilon} T^2.
\end{align}

Then, because at least one member in a sequence must perform as well as the mixture, we know that $J(\pi_E) - J(\pi)\leq \Bar{\epsilon} T^2$, where $\pi \in \{\pi_1, \dots, \pi_N\}$ is the member with the lowest validation error.
\end{proof}

\subsection{Proof of Theorem \ref{thm:filter-nr}}
\label{pf:thm:filter-nr}
\begin{proof}
    We define $L_i$ and $\Bar{\epsilon}$ as before, i.e.
    \begin{equation}
    L_i(\pi, f) = \frac{1}{T} \sum_{t=0}^T \mathbb{E}_{s \sim \rho_t^{E}}[\mathbb{E}_{a \sim \pi(s)}[\mathbb{E}_{\xi \sim \pi_i|s, a}[\sum_{\tau = t}^T f(s_{\tau}, a_{\tau})]]],
\end{equation}
\begin{equation}
    \Bar{\epsilon} = \frac{1}{NT} \sum_{i=1}^N L_i(\pi_i, f_i) - \min_{\pi \in \Pi}\frac{1}{NT} \sum_{i=1}^N L_i(\pi, f_i) \in [-1, 1].
\end{equation}
Additionally, we define an average regret measure for the adversary:
\begin{equation}
    \Bar{\delta} = \max_{f \in \mathcal{F}_r}\frac{1}{NT} \sum_{i=1}^N L_i(\pi_i, f) - \frac{1}{NT} \sum_{i=1}^N L_i(\pi_i, f_i) \in [-1, 1].
\end{equation}
Note that
\begin{equation}
    \Bar{\epsilon} + \Bar{\delta} = \frac{1}{NT} (\max_{f \in \mathcal{F}_r} \sum_{i=1}^N L_i(\pi_i, f) - \min_{\pi \in \Pi} \sum_{i=1}^N L_i(\pi, f_i)).
\end{equation}

Proceeding as before,
\begin{align}
    J(\pi_E) - J(\Bar{\pi}) &\leq \max_{f \in \mathcal{F}_r} \frac{1}{N} \sum_{i=1}^N  J(\pi_i, f) - J(\pi_E, f) \\
    &= \max_{f \in \mathcal{F}_r} \frac{1}{N} \sum_{i=1}^N T (L_i(\pi_i, f) - L_i(\pi_E, f)) \\
    &\leq \frac{1}{N} \sum_{i=1}^N \max_{f_i \in \mathcal{F}_r} T (L_i(\pi_i, f_i) - L_i(\pi_E, f_i)) \\
    &\leq \frac{1}{N} \min_{\pi \in \Pi} \sum_{i=1}^N \max_{f_i \in \mathcal{F}_r} T (L_i(\pi_i, f_i) - L_i(\pi, f_i)) \\
    &= \frac{1}{N} \sum_{i=1}^N T (\Bar{\epsilon}T + \Bar{\delta}T) \\
    &= (\Bar{\epsilon} + \Bar{\delta}) T^2.
\end{align}
\end{proof}

\subsection{Proof of Theorem \ref{thm:filter-lb}}
\label{pf:thm:filter-lb}
\begin{proof}
We again consider the \textsc{Cliff} MDP. As before, assume that the expert always takes $a_1$, $r(s, a) = -\mathds{1}_{s_x} -\mathds{1}_{a_2}$, and that $\mathcal{F}_r = \{r\}$.

Let $\pi$ be a policy that takes $a_2$ in $s_0$ with prob. $\epsilon T$ and $a_1$ with prob. 1 everywhere else. Thus, on a $\frac{T-1}{T}$ fraction of the rollouts, there is no difference between the learner on the expert. On the $\frac{1}{T}$ fraction of rollouts that start from $s_0$, the learner diverges from the expert for the entire horizon with probability $\epsilon T$, so the discriminator can penalize it $\epsilon T^2$ on average. Putting it all together, $\Bar{\epsilon} = \frac{1}{T}[\frac{T-1}{T}(0) + \frac{1}{T}(\epsilon T)] = \epsilon$. The outer $\frac{1}{T}$ comes from the average over timesteps in the payoff. 

On rollouts, the learner would have an $\epsilon T$ chance of paying a cost of 1 for the rest of the horizon (as they always start at $s_0$), leading to a lower bound of $J(\pi_E, r) - J(\pi, r) = \epsilon T^2 \geq \Omega(\epsilon T^2)$. 
\end{proof}
We note that if we started on the true start-state distribution ($\alpha = 0$), we would instead get an $\Bar{\epsilon} = \epsilon T$ and therefore a bound linear in the horizon, recovering the lower bound results in \citet{swamy2021moments}.

\begin{theorem}
    The trajectory-based sampling procedure implied by Equation \ref{eq:traj} is lower variance than the suffix-based sampling procedure implied by Equation \ref{eq:suffix}.
    \label{thm:var}
\end{theorem}

\subsection{Proof of Theorem \ref{thm:var}}
\label{pf:thm:var}
\begin{proof}
First, let us explicitly define the sampling procedure implied by each of the above. At step $t$:
\begin{itemize}
    \item Equation \eqref{eq:suffix}: Sample a state-action pair from the  expert visitation distribution at timestep $t$. Reset the learner to this state. Execute the sampled action and then roll out the learner for $T - t$ timesteps, adding up the reward function along this suffix. Sample a state-action pair from the  expert visitation distribution at timestep $t$. Reset the learner to this state. Roll out the learner for $T - t + 1$ timesteps, adding up the reward function along this suffix. Record the difference of these two suffix sums.
    \item Equation \eqref{eq:traj}: Roll out the current learner policy for $t$ timesteps. Use the sample at timestep $t$ for evaluating the reward function. Sample a state-action pair from the  expert visitation distribution at timestep $t$. Use this sample for evaluating the reward function. Record the difference of these two single-step evaluations.
\end{itemize}

We consider two settings:
\begin{enumerate}
    \item Total independence between timesteps: $\forall t \in [T]$, $\text{Var}(r(s_t, a_t)) = \sigma^2$.
    \item Total dependence (determinism) between timesteps: $\text{Var}(r(s_0, a_0)) = \sigma^2$, $\forall t \in [1, T]$, $r(s_t, a_t) = r(s_0, a_0)$.
\end{enumerate}
Let's begin with Case 1. Equation \eqref{eq:suffix} has variance 
\begin{equation}
    \sum_t^T (T-t) \sigma^2 + (T-t) \sigma^2 = T(T-1) \sigma ^2,
\end{equation}
while Equation \eqref{eq:traj} has variance
\begin{equation}
    \sum_t^T \sigma ^2 + \sigma ^ 2 = 2 T \sigma^2.
\end{equation}
Observe that $2 T \sigma^2 < T(T-1) \sigma ^2$ to complete this case. Similarly, for Case 2, Equation \eqref{eq:suffix} has variance 
\begin{equation}
    \sum_t^T (T-t)^2 \sigma^2 + (T-t)^2 \sigma^2 = \frac{\sigma^2 (T)(T+1)(2T+1)}{3},
\end{equation}
and Equation \eqref{eq:traj} has variance
\begin{equation}
    \sum_t^T \sigma ^2 + \sigma ^ 2 = 2 T \sigma^2.
\end{equation}
The former variance is again greater than the latter, completing this case and the proof.
\end{proof}

\clearpage

\section{Dual Algorithms}
\label{app:dual}

\begin{algorithm*}[t]
\begin{algorithmic}
\STATE {\bfseries Input:} Sequence of expert visitation distributions $\rho_{E}^1 \dots \rho_{E}^T$, Policy class $\Pi$, Reward class $\mathcal{F}_r$
\STATE {\bfseries Output:} Sequence of trained policies $\pi = \pi^{1:T}$
\STATE Initialize $\pi_1 = (\pi_0^1, \dots, \pi_0^T) \in \Pi^T$,
\FOR{$i$ in $1 \dots N$}
\STATE \algcomment{Use any no-regret algo to pick $f$, e.g. FTRL}
\STATE $f_{i} \gets \underset{f \in \mathcal{F}_r}{\argmax } \, J(\pi_E, f) - J(\text{Unif}(\pi_{1:i}), f)+ R(f)$
\STATE \algcomment{Perform an expert-competitive response via PSDP}
\STATE $\pi^T_{i+1} = \argmax_{\pi \in \Pi} \mathbb{E}_{s \sim \rho_E^T}[f_i(s, \pi(s))]$
\FOR{$t$ in $T \dots 2$}
\FOR{$j=1$ to $M$}
\STATE Sample a random start state $s_{t-1} \sim \rho_{E}^{t-1}$.
\STATE Execute a random action $a_{t-1} \sim \text{Unif}(\mathcal{A})$ in $s_{t-1}$.
\STATE Follow $\pi^{t:T}_{i+1}$ until the end of the horizon.
\STATE $\mathcal{D}_t \gets \mathcal{D}_t \cup \{(s_t, a_t, s_{t+1:T}, a_{t+1:T})\}$
\ENDFOR
\STATE \algcomment{Use any cost-sensitive classification algo to pick $\pi^{t-1}_{i+1}$}
\STATE \begin{equation}
    \pi^{t-1}_{i+1} = \argmax_{\pi \in \Pi} \mathbb{E}_{s_{t-1} \sim \rho_E^{t-1}}\left[f_i(s_{t-1}, \pi(s_{t-1})) +  \mathbb{E}_{\mathcal{D} \vert s_{t-1}, \pi(s_{t-1})} \left [ \sum_{\tau=t}^T f_i(s_{\tau}, \pi_{\tau}(s_{\tau})) \right]\right]
\end{equation}
\ENDFOR
\ENDFOR
\STATE {\bfseries Return } $\Bar{\pi}_{1:T}$, uniform mixture of $\pi^{1:T}_{i}$.
\end{algorithmic}
\caption{\texttt{MMDP} (Moment Matching by Dynamic Programming): Dual \label{alg:mmdp-dual}}
\end{algorithm*}
\begin{algorithm*}[t]
\begin{algorithmic}
\STATE {\bfseries Input:} Sequence of expert visitation distributions $\rho_{E}^1 \dots \rho_{E}^T$, Policy class $\Pi$, Reward class $\mathcal{F}_r$
\STATE {\bfseries Output:} Sequence of trained policies $\pi = \pi^{1:T}$
\STATE Initialize $\pi_1 = (\pi_0^1, \dots, \pi_0^T) \in \Pi^T$,
\FOR{$i$ in $1 \dots N$}
\STATE \algcomment{Use any no-regret algo to pick $f$, e.g. FTRL}
\STATE $f_{i} \gets \underset{f \in \mathcal{F}_r}{\argmax } \, J(\pi_E, f) - J(\text{Unif}(\pi_{1:i}), f)+ R(f)$
\STATE \algcomment{Perform an expert-competitive response via NRPI}
\STATE Set $\mathcal{D}_{i} = \{\}$, $\pi_i^1 = \pi_{i-1}$.
\FOR{$j=1$ to $M$}
\STATE Sample random time $t \sim \text{Unif}([0, T])$ and start state $s_t \sim \rho_{E}^t$.
\STATE Execute a random action $a_t \sim \text{Unif}(\mathcal{A})$ in $s_t$.
\STATE Follow $\pi_{i}^j$ until the end of the horizon.
\STATE $\mathcal{D}_{i} \gets \mathcal{D}_{i} \cup \{(s_t, a_t, t, \hat{Q}_t = \sum_{\tau =t}^T f_{i}(s_{\tau}, a_{\tau}))\}$
 \STATE \algcomment{Run any no-regret CSC algorithm on $\mathcal{D}_{i}$ to produce new $\pi_i^{j+1}$, e.g. FTRL:}
\STATE Optimize \begin{equation}
    \pi_i^{j+1} \gets \argmax_{\pi \in \Pi} \mathbb{E}_{s \sim \mathcal{D}, a \sim \pi(s)}[\mathbb{E}[\hat{Q}_t|s_t = s, a_t = a]] + H(\pi).
\end{equation}
\ENDFOR
\STATE Select $\pi_i$ as best of $\pi_i^{1:M}$ on validation data.
\ENDFOR
\STATE {\bfseries Return } $\Bar{\pi}$, uniform mixture of $\pi_{i}$.
\end{algorithmic}
\caption{\texttt{NRMM} (No Regret Moment Matching): Dual \label{alg:nrmm-dual}}
\end{algorithm*}

We now present the \textit{dual} (i.e. no-regret over rewards rather than over policies) variants of the algorithms in the preceding section.
We first describe our proof strategy for general two-player zero-sum games before specializing to inverse RL.

\subsection{Two Player Zero-Sum Games with Relative Best Responses}

\begin{definition}[Relative Best Response] We say that an oracle $\mathcal{A}_y: \mathcal{X} \to \mathcal{Y}$ satisfies \texttt{RBR[$y_E, \epsilon$]}, if, $\forall x \in \mathcal{X}$, we have that
\begin{equation}
    \ell(x, \mathcal{A}_y(x)) - \ell(x, y_E) \leq \epsilon.
\end{equation}
\end{definition}

\begin{theorem}
    Consider a two-player zero-sum game $\max_{x \in \mathcal{X}} \min_{y \in \mathcal{Y}} \ell(x, y)$ with payoff $\ell$ that is concave in $x$ and convex in $y$. Let $y_E \in \mathcal{Y}$. Given access to a no-regret online convex optimization algorithm $\mathcal{A}_x$ over $\mathcal{X}$ and a \texttt{RBR[$y_E, \epsilon$]} oracle $\mathcal{A}_y$, we are able to compute an average iterate $\bar{y}$ in $N$ rounds such that
    \begin{equation}
        \max_{x \in \mathcal{X}} \ell(x, \bar{y}) - \ell(x, y_E) \leq \epsilon + \frac{\text{Reg}_{\mathcal{X}}(N)}{N}.
    \end{equation}
    \label{thm:rbr}
\end{theorem}

    \begin{proof}
    Define $\ell_i(x) = \ell(x, y_i) - \ell(x, y_E)$. We set $x_n = \mathcal{A}_x(\ell_{1:n-1})$ and $y_n = \mathcal{A}_y(x_n)$.
    
    By the convexity of $\ell$ in $y$ and Jensen's inequality, we have that
    \begin{align}
        \max_{x \in \mathcal{X}} \ell(x, \bar{y}) - \ell(x, y_E) \leq \max_{x \in \mathcal{X}} \frac{1}{N} \sum_{n=1}^N \ell(x, y_n) - \ell(x, y_E).
    \end{align}

        To prove our original claim, it is sufficient to upper-bound
    \begin{align}
         \max_{x \in \mathcal{X}} \frac{1}{N} \sum_{n=1}^N \ell_t(x).
    \end{align}
    From the relative best response property, we directly have that %
    \begin{equation}
        \forall n \in [N], \ell(x_n, y_n) - \ell(x_n, y_E) \leq \epsilon \Rightarrow \forall n \in [N], \ell_n(x_n) \leq \epsilon.
    \end{equation}
   By the definition of regret of $\mathcal{A}_x$, we have that %
    \begin{equation}
        \max_{x \in \mathcal{X}} \frac{1}{N} \sum_{n=1}^N \ell_n(x) - \ell_n(x_n) \leq \frac{\text{Reg}_{\mathcal{X}}(N)}{N}.
    \end{equation}
    Rearranging terms, this tells us that
    \begin{equation}
        \max_{x \in \mathcal{X}} \frac{1}{N} \sum_{n=1}^N \ell_n(x)  \leq \frac{\text{Reg}_{\mathcal{X}}(N)}{N} + \frac{1}{N} \sum_{n=1}^N \ell_n(x_n) \leq \frac{\text{Reg}_{\mathcal{X}}(N)}{N} + \epsilon.
    \end{equation}
    where the last inequality comes from the definition of \texttt{RBR[$y_E, \epsilon$]}. This completes the proof. We note that via the no-regret property of $\mathcal{A}_x$, the first term tends to 0 as $N \to \infty$.
    \end{proof}

\subsection{Inverse Reinforcement Learning with Expert-Competitive Responses}
Our proofs will assume access to an efficient oracle to perform an \textit{expert-competitive response} -- a relative best response where the performance of the learner's policy is measured relative to $\pi_E$.
\begin{definition}[Expert-Competitive Response] We say that a reinforcement learning algorithm $\mathcal{A}: \mathcal{F}_r \to \Pi$ satisfies \texttt{ECR[$\epsilon T^2$]}, if, $\forall r \in \mathcal{F}_r$, we have that
\begin{equation}
    J(\pi_E, r) - J(\mathcal{A}(r), r) \leq \epsilon T^2.
\end{equation}
\end{definition}
For an arbitrary $\epsilon > 0$, both PSDP and NRPI satisfy \texttt{ECR[$\epsilon T^2$]} when run with the expert demonstrations as their reset distribution. Note that for MMDP, $\epsilon$ refers to the cost-sensitive classification error of the inner PSDP loop while for NRMM, $\epsilon$ refers to the regret of the online cost-sensitive classification algorithm.

We now provide a shared policy performance bound proof for both dual algorithms.  %

\begin{theorem}{\textbf{\texttt{MMDP} / \texttt{NRMM} Dual Variant Upper Bound:}} Assume we have access to a no-regret online linear optimization algorithm $\mathcal{A}_f$ over $\mathcal{F}_r$ and an \texttt{ECR[$\epsilon T^2$]} oracle $\mathcal{A}_{\pi}$ over $\Pi$. Define $\ell_i(f) = \frac{1}{T}(J(\pi_E, f) - J(\pi_i, f))$. Set $f_i = \mathcal{A}_f(\ell_{1:i-1})$ and $\pi_i = \mathcal{A}_{\pi}(f_i)$. Then, if we let $\bar{\pi}$ denote the uniform mixture over computed policies, we have that
    \begin{equation}
        J(\pi_E, r) - J(\Bar{\pi}, r) \leq \Bar{\delta} T + \epsilon T^2,
    \end{equation}
    where $\Bar{\delta} = \frac{1}{N} \sum_{n=1}^N \delta_n$ is the average regret of the reward selection algorithm used.
    \label{thm:dual}
\end{theorem}
\begin{proof}
Setting $\mathcal{X} = \mathcal{F}_r$, $\mathcal{Y} = \Pi$, $\ell(r, \pi) = \frac{1}{T} J(\pi, r)$, $y_E = \pi_E$, $\epsilon = \epsilon T$, and applying Theorem \ref{thm:rbr} gives us the claim.

\end{proof}

Observe that both Algorithm \ref{alg:mmdp-dual} and \ref{alg:nrmm-dual} fit into the template defined in the theorem statement. By inspecting our above algorithms, it is clear that our per-iteration sample complexity analysis applies as written so we do not rehash the details. We note however that in contrast to the overall polynomial sample complexity of the primal version of \texttt{MMDP}, we do have to pay linearly in the number of no-regret iterations for the dual version. Similarly, for the dual version of NRMM, we have to pay for the number of no-regret iterations over both strategy spaces.

\clearpage
\section{Experiments}
\label{app:exps}
We use Optimistic Adam \citep{daskalakis2017training} for all policy and discriminator optimization, taking advantage of its speed and last-iterate convergence properties. We use gradient penalties \citep{gulrajani2017improved} to stabilize our discriminator training for all algorithms. Our policies, value functions, and discriminators are all 2-layer ReLu networks with a hidden size of 256. Each outer loop iteration lasts for 5000 steps of environment interaction. We sample 4 trajectories to use in the discriminator update at the end of each outer-loop iteration.

\subsection{PyBullet Tasks}
For the PyBullet tasks (Walker, Hopper, HalfCheetah), we use the Soft Actor Critic \citep{haarnoja2018soft} implementation provided by \citet{raffin2019stable} for policy optimization for both the expert and the learner. We use the hyperparameters in Table \ref{table:stbsln3params} for all experiments. We train behavioral cloning for 100,000 steps.

\begin{table}[h]
\begin{center}
\begin{small}
\begin{sc}
\setlength{\tabcolsep}{2pt}
\begin{tabular}{lccccccccccr}
\toprule
 Parameter & Value \\
\midrule
 buffer size & 300000 \\
 batch size & 256 \\
 $\gamma$ & 0.98 \\
 $\tau$ & 0.02 \\
 Training Freq. & 64 \\
 Gradient Steps & 64 \\
 Learning Rate & Lin. Sched. 7.3e-4 \\
 policy architecture & 256 x 2 \\
 state-dependent exploration & true \\
 training timesteps & 1e6 \\
\bottomrule
\end{tabular}
\end{sc}
\end{small}
\end{center}
\caption{\label{table:stbsln3params} Expert and learner hyperparameters for SAC.}
\end{table}

We use $\alpha=0.5$ for both variants of \texttt{FILTER} as we found it to perform better than $\alpha=1$.

For our discriminator, we start with a learning rate of $8e-3$ and decay it linearly over outer-loop iterations.

\subsection{D4RL Tasks}
For the D4RL tasks (both large \texttt{antmazes}), we use the data provided by \citet{fu2020d4rl} as our expert demonstrations. We give all algorithms access to goal information by appending it to the observation. This helps explain why our behavioral cloning baseline significantly out-performs previously published results and might be of independent interest to the Offline RL community. \footnote{We found that on the small and medium mazes, a properly tuned implementation of BC was able to achieve scores upwards of 70.} Importantly, we did not filter the data down whatsoever as in the "\%-BC" approach of \citet{chen2021decision}, so our algorithms are all truly reward-free.

For our policy optimizer, we build upon the TD3+BC implementation of \citet{fujimoto2021minimalist} with the default hyperparameters. For behavioral cloning, we run the optimizer for 500k steps while zeroing out the component of the actor update that depends on rewards. 

For \texttt{MM} and \texttt{FILTER}, we pre-train the policy with 10,000 steps of behavioral cloning. We use a dual replay buffer strategy, similar to that of \citet{hester2018deep, reddy2019sqil, swamy2021moments}. One buffer contains expert demonstrations while the other contains learner rollouts. We sample a batch from one with equal probability for each policy update. For samples from the expert buffer, we use the current discriminator to impute rewards and use the BC regularizer term. For samples from the learner's buffer, we use the recorded discriminator values and turn off the BC regularizer. We use $\alpha=1$ for \texttt{FILTER} (i.e. \texttt{NRMM}).

For our discriminator, we start with a learning rate of $8e-4$ and decay it linearly over outer-loop iterations.

\end{document}